\documentclass[pdflatex,sn-mathphys-num,onecolumn]{sn-jnl}
\usepackage{graphicx}%
\usepackage{multirow}%
\usepackage{amsmath,amssymb,amsfonts}%
\usepackage{amsthm}%
\usepackage{mathrsfs}%
\usepackage[title]{appendix}%
\usepackage{xcolor}%
\usepackage{textcomp}%
\usepackage{manyfoot}%
\usepackage{booktabs}%
\usepackage{algorithm}%
\usepackage{algorithmicx}%
\usepackage{algpseudocode}%
\usepackage{listings}%
\usepackage{tikz}
\usetikzlibrary{positioning, arrows.meta, calc}
\usepackage{enumitem}
\usepackage{caption}
\usepackage{amsthm}
\usepackage{tabularx, booktabs}

\usepackage{mathtools} 
\usepackage{verbatim}
\usepackage{pifont}

\theoremstyle{thmstyleone}%
\theoremstyle{thmstyletwo}%

\theoremstyle{thmstylethree}%
\usepackage{tcolorbox}

\theoremstyle{plain}
\newtheorem{thm}{Theorem}[section]
\newtheorem{lem}[thm]{Lemma}
\newtheorem{assump}{Assumption}
\newtheorem{coro}[thm]{Corollary}

\usepackage{mathtools}
\newtcolorbox[list inside=prompt,auto counter,number within=section]{prompt}[1][]{
    colbacktitle=black!60,
    coltitle=white,
    fontupper=\footnotesize,
    boxsep=5pt,
    left=0pt,
    right=0pt,
    top=0pt,
    bottom=0pt,
    boxrule=1pt,
    title={#1},
    #1, 
}
\newtcolorbox[list inside=CASE, auto counter, number within=section]{CASE}[1][]{
    colbacktitle=black!60,
    coltitle=white,
    fontupper=\footnotesize,
    boxsep=5pt,
    left=0pt,
    right=0pt,
    top=0pt,
    bottom=0pt,
    boxrule=1pt,
    title={#1},
    #1, 
}

\begin{document}
\title[Article Title]{CMCTS: A Constrained Monte Carlo Tree Search Framework for Mathematical Reasoning in Large Language Model}

\author[1]{\fnm{Qingwen} \sur{Lin}}\email{qingwen\_lin@foxmail.com}

\author[1]{\fnm{Boyan} \sur{Xu}}\email{hpakyim@gmail.com}
\author[2]{\fnm{Guimin} \sur{Hu}}\email{guhu@di.ku.dk}

\author[1]{\fnm{Zijian} \sur{Li}}\email{leizigin@gmail.com}

\author[1,3]{\fnm{Zhifeng} \sur{Hao}}\email{haozhifeng@stu.edu.cn}

\author[4]{\fnm{Keli} \sur{Zhang}}\email{zhangkeli1@huawei.com}

\author*[1,5]{\fnm{Ruichu} \sur{Cai}}\email{cairuichu@gmail.com}\equalcont{Corresponding author}

\affil*[1]{\orgdiv{School of Computer Science}, \orgname{Guangdong University of Technology}}
\affil[2]{\orgdiv{Department of Computer Science }, \orgname{University of Copenhagen}}
\affil[3]{\orgdiv{College of Science}, \orgname{Shantou University}}

\affil[4]{\orgdiv{Huawei Noah’s Ark Lab}}

\affil[5]{\orgdiv{Peng Cheng Laboratory}}

\abstract{
This paper introduces the Constrained Monte Carlo Tree Search (CMCTS) framework to enhance the mathematical reasoning capabilities of Large Language Models (LLM). By incorporating a constrained action space, Process Reward Model (PRM), and partial order rules, CMCTS effectively addresses the limitations of existing MCTS methods in terms of state space diversity and action selection rationality. Specifically, during the expansion phase, CMCTS restricts action sampling to a predefined constrained action set to increase candidate state diversity. In the simulation phase, it introduces partial order rules and PRM to optimize action selection and prevent unreasonable state transitions. Experimental results show that CMCTS performs outstandingly across multiple mathematical reasoning benchmarks. Under a zero-shot setting, a 7B-parameter model achieves an average accuracy of 83.4\%, surpassing the 72B baseline model by 4.8\%. Ablation studies demonstrate that each component of the framework is crucial for performance improvement, and their combined use fully leverages their respective strengths. Overall, the CMCTS framework provides an effective approach to enhancing LLM mathematical reasoning capabilities, supported by theoretical analysis, and offers novel insights for future reasoning tasks.}


\maketitle

\section{Introduction}\label{sec:intro}

Improving the reasoning ability, especially the
mathematical reasoning ability, occupies a central
position in current large language models (LLM)
research. The Chain of Thought (CoT)\cite{wang2022self} technique has emerged as a mainstream solution to enhance LLM’ reasoning ability in a step-by-step manner. Recently, the generation of Long Chains of Thought (Long COT)\cite{guo2025deepseek,zhong2024evaluation,team2024qwq,team2025kimi} has led to significant performance improvements. 

Compared to CoT, Long CoT introduces two critical enhancements that significantly improve the reasoning ability of large language models, particularly in complex mathematical tasks:
(i) \textit{Semantic diversity.} Long CoT expands the model's exploration space by enabling the generation of diverse intermediate steps, alternative solution strategies, and multi-level decompositions\cite{guo2025deepseek,zhong2024evaluation,team2024qwq,team2025kimi}. This diversity allows the model to avoid early commitment to a potentially suboptimal reasoning path and instead consider multiple plausible approaches in parallel.
(ii) \textit{Reflective self-verification.} Unlike CoT, which typically proceeds in a fixed, linear fashion, Long CoT encourages the model to revisit earlier steps, check the consistency or correctness of intermediate results, and revise its reasoning path as needed.

While these capabilities make Long CoT a powerful paradigm, they are difficult to elicit reliably in practice. Without reinforcement learning or instruction tuning, most LLM-such as Qwen or LLaMA-struggle to spontaneously generate diverse, self-correcting reasoning chains in zero-shot settings~\cite{gandhi2025cognitive, yue2025does}. As a result, the generated solutions often collapse into shallow, repetitive patterns lacking depth and verification. This motivates the need for external control mechanisms to guide the model toward more structured and reliable reasoning behavior.

In reasoning-strategy-optimized methods, tree search algorithms are considered highly effective in enhancing the mathematical reasoning capabilities of LLM \cite{yao2023tree, sun2024beats, hao2023reasoning}, particularly Monte Carlo Tree Search (MCTS) and its numerous variants \cite{hao2023reasoning, feng2023alphazero, qi2024mutual}. 
MCTS constructs chains of thought through four key stages: Selection, Expansion, Simulation, and Backpropagation. During this process, MCTS leverages the model itself as a reward function to guide state transitions, enabling the generation of more optimal reasoning paths. Compared with traditional CoT methods, \cite{wei2022chain}, MCTS is a method that can significantly enhance model performance by generating higher-quality COT without requiring additional training.

However, despite the significant improvement in CoT reasoning performance that MCTS methods can achieve, they still fail to enable LLM to generate high-quality long COT. This is primarily due to the following limitations: 1) homogenization of the state space; 2) lack of effective selection of key states. Existing MCTS relies on model-generated actions for state transitions. For models that have not been trained with reinforcement learning, generating actions that can induce sufficient diversity is challenging \cite{gandhi2025cognitive, yue2025does}. Moreover, LLM as reward models often fail to accurately evaluate the quality of intermediate reasoning steps \cite{zhang2024rest, luo2024improve}, preventing them from transitioning to appropriate states.
Our theoretical analysis of the MCTS reasoning process has led to two key findings: 1) Enhancing semantic richness can exponentially reduce the error compared to the true state; 2) Even with diverse states, poor state selection can negate the advantages brought by this diversity. These findings are consistent with previous studies \cite{brown2024large,li2024common,snell2024scaling}, further indicating that sampling a diverse state space in mathematical reasoning can significantly improve performance.

\begin{figure}[ht!]
\centering
\includegraphics[width=1\textwidth]{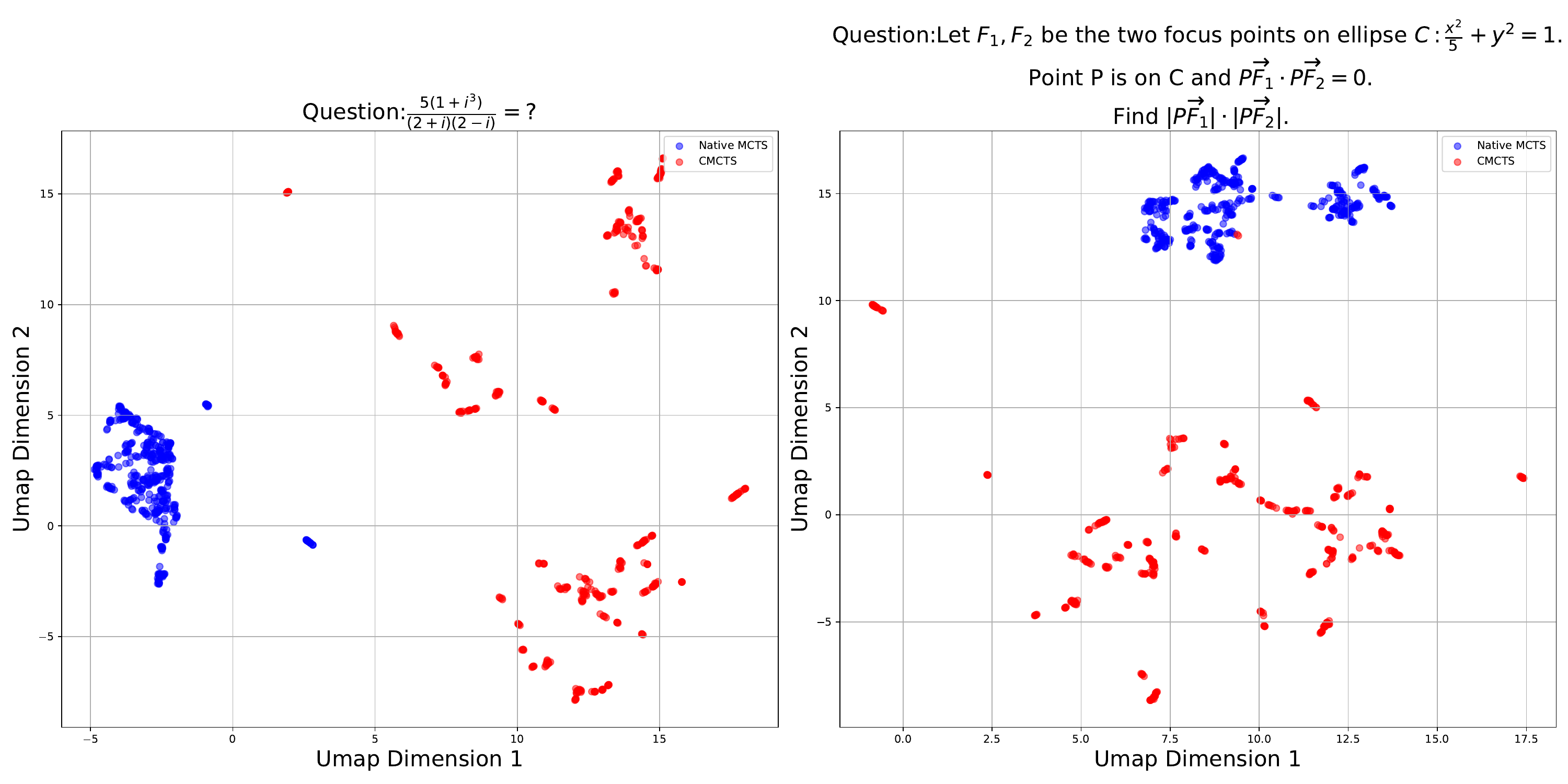}
\caption{The visualization of different states generated by Native-MCTS and CMCTS when reasoning about the same problem. }\label{fig:visual}
\end{figure}

From this perspective, we have developed the Constrained Monte Carlo Tree Search (CMCTS) framework, as shown in Figure \ref{fig:main}. First, during the expansion phase, we partition the action space into four disjoint constrain subsets: understand, reflect, code, and summary. By using constrained action-space sampling, we force the model to sample in regions of the state space that are difficult to sample otherwise, thereby enhancing state-space diversity. Figure \ref{fig:visual} illustrates that our proposed CMCTS generates a more diverse state space compared to Native-MCTS in the state spaces. Second, we introduce partial order rules and a PRM~\cite{li2024process}, aiming to enhance the rationality of the sampled actions.
Partial order rules constrain the order of actions to ensure logical and interpretable reasoning, while PRM scores different actions using pre-trained models to help select the optimal ones. These methods can be used individually or in combination to improve action-selection accuracy. Extensive experiments show that the CMCTS framework performs well in mathematical reasoning benchmarks. The 7B model equipped with CMCTS achieves an average accuracy of 83.4\%, which is 4.8\% higher than that of the 72B CoT model in terms of mathematical reasoning performance. Moreover, on the Math dataset, our method shows a significant performance improvement compared to other MCTS variants.

\begin{enumerate}
    \item Restrict model sampling to a predefined constrain action set during the expansion phase to increase candidate-state diversity and enhance MCTS performance.
    \item Design partial order rules and a PRM during the simulation phase to optimize action selection.
    \item Theoretically describe the key factors for improving the MCTS algorithm and explain how the proposed method enhances the performance of the MCTS algorithm.
    \item Enable the 7B model equipped with the CMCTS framework to generate long-text responses and outperform the 72B model across multiple mathematical reasoning benchmarks without additional training.
\end{enumerate}

\section{Related Work}
LLM have recently become the focus of the academic community due to their remarkable performance. Typical examples include the GPT series \cite{brown2020language,chen2021evaluating,nakano2021webgpt,achiam2023gpt}, the LLaMA series \cite{touvron2023llama,taori2023alpaca,young2024yi,dubey2024llama}, the Gemma series \cite{team2024gemma_a,team2024gemma_b}, the RWKV series \cite{peng2025rwkv,peng2024eagle,peng2023rwkv}, the DeepSeek series \cite{bi2024deepseek,liu2024deepseek,guo2025deepseek}, and the Qwen series \cite{yang2024qwen2b,yang2024qwen2a}. Among the many capabilities of LLM, logical reasoning, especially mathematical reasoning, is considered a core competency.

To enhance the mathematical reasoning abilities of LLM, researchers have designed various methods. For example, reinforcement learning based on human preferences \cite{li2023reinforcement,ouyang2022training,rafailov2024direct,ethayarajh2024kto}, using specialized mathematical models equipped with mathematical knowledge \cite{shao2024deepseekmath,yang2024qwen2b,ying2024internlm}, and generating long CoT (Chain of Thought) through reinforcement learning rewards \cite{guo2025deepseek,zhong2024evaluation,team2024qwq,team2025kimi} have been employed. Additionally, knowledge distillation from large models to improve the performance of smaller models \cite{lin2024large,guo2025deepseek,team2024gemma_b} and self-distillation of smaller models to enhance performance \cite{feng2023alphazero,chen2024alphamath,guan2025rstar} have also been adopted. These methods have achieved impressive results, but training LLM is typically costly and requires substantial computational resources. Therefore, when computational resources are limited, researchers seek to optimize reasoning strategies to enhance model performance. For these algorithms, we can generally categorize them into two types: \textbf{state expansion} and \textbf{step evaluation}.

\textbf{State Expansion:} LLM inherently possess the ability to generate highly diverse states, but without additional training, they struggle to sample these states \cite{gandhi2025cognitive, yue2025does}. Among these hard-to-sample states, there may lie the keys to solving problems, such as self-verification and reflection commonly used in long-CoT. Therefore, additional prior constraints are needed to ensure that LLM sample these states. Thus, the main research direction of state-expansion-based methods is to use specific search algorithms to enable LLM to generate states that cannot be directly generated. The pioneering work in this area is COT \cite{wei2022chain}, which uses prompts to induce step-by-step reasoning from LLM, thereby significantly enhancing the model's reasoning performance. RAG-related work retrieves world knowledge to constrain LLM to correctly sample world knowledge \cite{jiang2023active,xiong2024benchmarking,yu2024rankrag,zhao2025medrag}. Meta-Prompt technology \cite{guoconnecting} and NLRL \cite{feng2024natural} use past successful experiences to constrain the model to sample appropriate states based on historical experience.React \cite{yao2022react} and Rethink \cite{schwarzschild2024rethinking} constrain the model to rethink and summarize the CoT during the reasoning process to derive new answers. LLM-Modulo \cite{kambhampati2024position} employs multi-agent mechanisms to endow LLM with additional planning and comprehension capabilities. LLM can also incorporate code as a verification tool to obtain precise computational states that are otherwise hard to generate \cite{kim2023language,zelikman2024self}.

These works have generally proven that constraining model sampling with certain prior instructions to generate otherwise hard-to-sample states can lead to significant performance improvements. Therefore, in Section \ref{sec:action}, we propose sampling commands from four predefined action sets-understand, reflect, code, and summary-during the expansion phase of MCTS. Compared with native-MCTS, which relies on LLM-generated actions, this approach enriches the state space of MCTS and enhances its state diversity.

\textbf{Step Evaluation:} Existing research on LLM has demonstrated that step-by-step reasoning processes can significantly enhance reasoning performance \cite{wei2022chain}. Clearly, evaluating and selecting rational reasoning steps can improve the reasoning capabilities of LLM.
AoT \cite{selLLM} periodically regenerates and restates the current problem state during the search process, reducing the model's need to focus on the entire context history and thereby alleviating its cognitive load. R-STAR \cite{qi2024mutual} employs two LLM for collaborative answer selection. For mathematical problems, the PRM \cite{zhangopenprm,zhang2025lessons,wang2024math} has been proven to better select the optimal answer by evaluating reasoning states. Agent-based work, through manually set rules, allows different agents to perform optimally at appropriate times to enhance reasoning performance \cite{leimacm,guo2024large,zhou2024star}. Among these works, tree-search-based methods, especially MCTS-related works \cite{hao2023reasoning,wu2024beyond,park2024ensembling,zhang2024llama}, have garnered widespread attention due to their excellent performance. TOT \cite{yao2023tree} views the problem-solving process as a search within a tree of thoughts, where each node represents a partial solution, and significantly improves reasoning performance through optimal state selection. Native-MCTS \cite{hao2023reasoning}, as the pioneering work of MCTS in LLM, uses the LLM itself as a reward model to select the best reasoning steps and achieves significant performance improvements. 

However, existing work has shown that using an LLM as a reward model is not an optimal setup \cite{zhang2024rest, luo2024improve} and may lead to sub-optimal paths. Therefore, we use PRM and partial order rules to ensure correct step evaluation.

\begin{figure}[ht]
\centering
\includegraphics[width=1\textwidth]{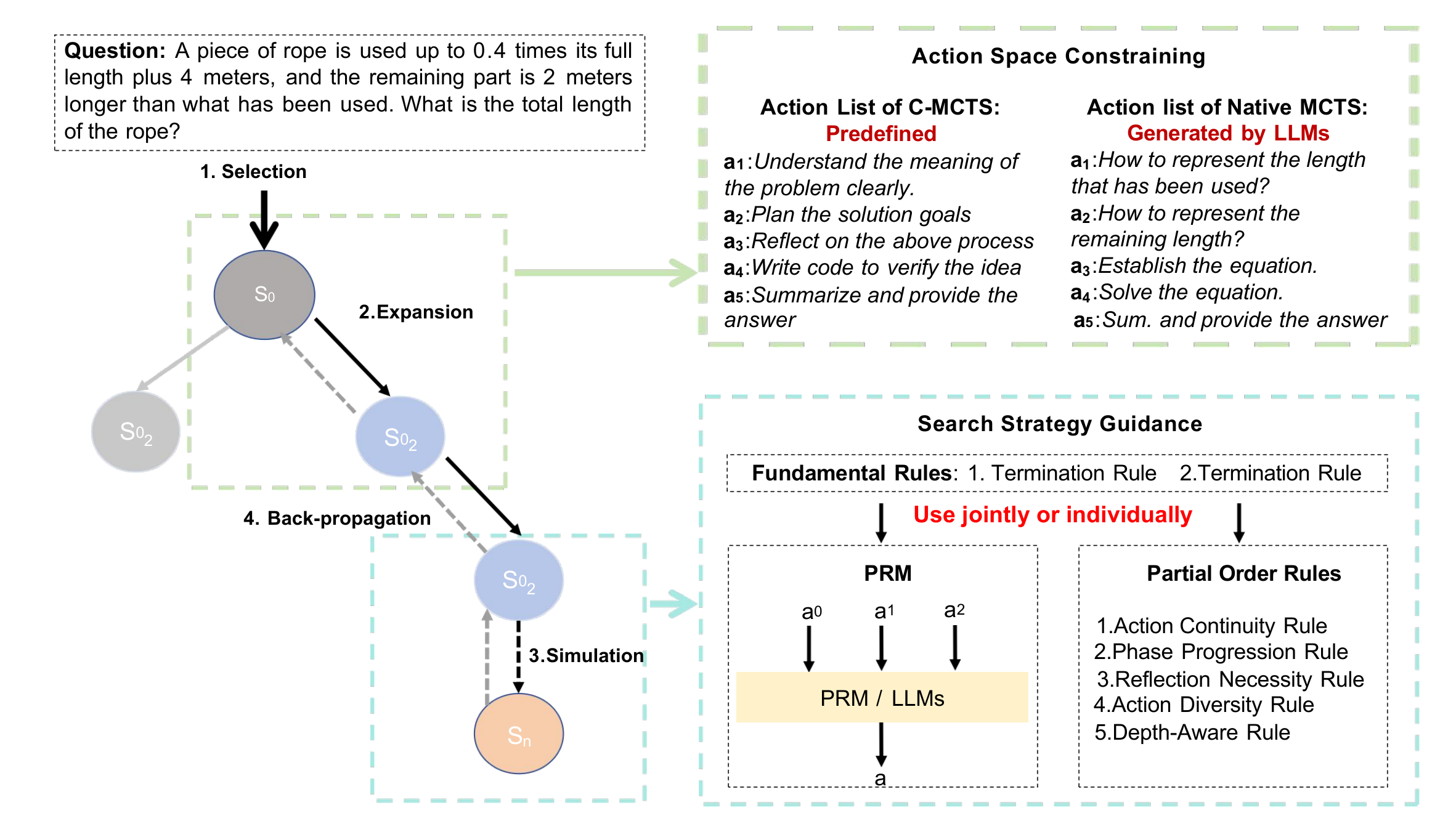}
\caption{A overview of our CMCTS framework  }\label{fig:main}
\end{figure}

\section{Method}

\subsection{The CMCTS Framework}\label{sec:framework}
Our method aims to address the challenge that existing MCTS methods cannot enable models like Qwen to generate Long  without additional training. As show in Figure \ref{fig:main}, the CMCTS framework consists of four main phases: Selection, Expansion, Simulation, and Back-propagation.

\subsubsection{Problem Definition}
We address the challenge of generating reasoning chains for mathematical problems by formulating it as a constrained Markov Decision Process (MDP). This approach aims to overcome the limitations of existing MCTS algorithms and generate more effective and higher-quality reasoning chains. The specific definition is as follows:

Let \( Q \) be the space of mathematical problems, and \( \Pi \) the family of LLM. Given a problem instance \( q \in Q \) and a reasoning strategy \( \pi \in \Pi \), the reasoning process can be modeled as an MDP tuple \((\mathcal{S}, \mathcal{A}, \mathcal{T}, \mathcal{R})\), where:

\begin{itemize}[leftmargin=0cm, itemindent=0cm]
    \item \(\mathcal{S}\) is the state space, representing the reasoning context at each step \( t \). The initial state \( s_0 \) is initialized by the problem instance \( q \).
    \item \(\mathcal{A}\) is the action space. In MCTS, actions are obtained before state transitions, and these actions determine the direction of state transitions.
    \item \(\mathcal{T}\) is the transition function, defined as \( s_{t+1} = \pi(s_t, a_{t}) \). Based on the current state \( s_t \) and the selected action \( a_{t} \), the LLM generates the next state \( s_{t+1} \).
    \item \(\mathcal{R}\) is the reward function, obtained through a reward model, providing supervision signals to enable optimal state transitions.
\end{itemize}




In this framework, forward reasoning is modeled as a sequence of transitions \( s_0 \rightarrow s_1 \rightarrow \dots \rightarrow s_t \), guided by actions \( a_1, a_2, \dots, a_t \). Each complete trajectory corresponds to a candidate solution to the problem. The ultimate goal of MCTS is to find one or more high-quality terminal states \(s_t\) that represent coherent and correct reasoning outcomes.

\subsubsection{Selection}

During the selection phase, we aim to identify the most promising part of the current search tree for further expansion. We employ the UCT (Upper Confidence bound for Trees) formula to balance exploration and exploitation. Specifically, the selection of each node \( s_t \) is based on the following formula:
\begin{equation}
\label{eq:select}
s_t = \arg\max_a \left( \text{mean}(R_{t}) + c_{uct} \sqrt{\frac{\log N_t}{N_t}} \right),
\end{equation}
where \( R_{t} \) represents the cumulative reward of the node, \( c_{uct} \) is a constant that balances exploration and exploitation, and \( N_t \) is the visit count of the current node.

\subsubsection{Expansion}

In the expansion phase, we sample actions from a predefined action set \(\mathcal{A}\) using a reward model to evaluate the potential of each action. To determine the most suitable action for the current state, we assess the reward of all potential actions using the reward model and select the action with the highest reward as \( a_t \). This selection process can be formally expressed as:
\begin{equation}
\label{eq:get_action}
a_t = \arg\max_{a_t^n \in \mathcal{A}} Q(s_t, a_t^n)
\end{equation}

After obtaining the action \( a_t \), we use the LLM \(\pi\) to generate multiple candidate states. To ensure diversity, we generate \( m \) candidate states and select the one with the highest V-value:
\begin{equation}
\label{eq:v-select}
s_{t+1} = \arg\max_{s_{t+1}^j \in \pi(s_t, a_{t}) } V(s_{t+1}^j),
\end{equation}
where \(\pi(s_t, a_{t}) \) is the set of candidate states generated by the LLM.\textbf{For the Expansion phase}, we propose in Section \ref{sec:action} that actions are sampled from a predefined constrained action space rather than being generated by the LLM.

\subsubsection{Simulation}
During the simulation phase, we repeatedly perform the expansion phase until a terminal node is reached, thereby generating a complete reasoning trajectory.
\begin{equation}
    s_{t+1} = \pi(s_t, a_t), \quad \text{until} \quad s_{t+1} \in \mathcal{S}_{\text{terminal}}
\end{equation}

\textbf{In the Simulation phase}, to help the model select proper actions and avoid transitions to inappropriate states, we introduce partial order rules and the Process Reward Model (PRM) as detailed in Section \ref{sec:guidance}. By combining human-like partial ordering with the PRM, our search strategy imposes constraints on the order of action execution. This ensures that actions are performed in a logical and interpretable sequence, guaranteeing high-quality reasoning throughout the process.

\subsubsection{Back-propagation}

In the back-propagation phase, we update the information of all nodes along the reasoning trajectory \([s_0, a_1, s_1, a_2, \dots, a_n, s_n]\), excluding the root node. We update the visit count \( N_t \) for each node and compute the new reward using the reward and V-value models:
\begin{equation}
\label{eq:reward}
\text{r}_t = Q(s_{t-1}, a_t) + V(s_t)
\end{equation}
The cumulative reward for each node is updated as:
\begin{equation}
\label{eq:cumulative_reward}
R_{t} \coloneqq \frac{ \sum_{N_t} \left( \sum_{i=t}^{T} \text{r}_t \right)}{N_t}
\end{equation}

The MCTS algorithm generates \( k \) subtrees corresponding to \( k \) reasoning trajectories through \( k \) iterations. After \( k \) iterations, we obtain \( k \) candidate reasoning trajectories \(\mathcal{P}_k\). We then design a voting-based aggregation algorithm to select the final answer from these trajectories. However, in some cases, there may be multiple answers with the highest frequency. In such cases, we can select the path with the highest reward as the final answer path. For any path \( p \) of length \( m \) in the path set \(\mathcal{P}_k\), we calculate the reward of its terminal node using Equation \ref{eq:reward} and select the node with the highest reward as the final answer:
\begin{equation}
\label{eq:final_answer}
p^* = \arg\max_{p \in \mathcal{P}_k}\space\text{r}_{p_{\text{terminal}}}
\end{equation}

\subsection{Action Space Constraining}\label{sec:action}

In traditional MCTS methods, the new action \( a_t \) is usually generated by LLM based on the current state \( s_{t-1} \). However, for LLM without reinforcement learning, the action sampling space is often limited \cite{gandhi2025cognitive, yue2025does}.

A good way to overcome this is to use predefined instructions to force LLM to sample areas that are otherwise hard to reach. Previous studies like COT \cite{wei2022chain} and React \cite{yao2022react} have used predefined instructions to induce LLM to generate hard-to-obtain responses. Specifically, during the Expansion phase of the CMCTS framework, we sample actions \( a_t \) from a predefined action space based on the current state \( s_{t-1} \). The LLM then combines \( a_t \) and \( s_{t-1} \) to generate a new state \( s_t \). We divide the action space into four disjoint subsets:

\[
\mathcal{A} = \left\{ \mathcal{A}^{\text{understand}}, \mathcal{A}^{\text{reflect}}, \mathcal{A}^{\text{code}}, \mathcal{A}^{\text{summary}} \right\}
\]

Each subset includes predefined  actions to guide model-state transitions:
\begin{itemize}
    \item \(\mathcal{A}^{\text{understand}}\): Problem understanding and information extraction. This action set guides the model to deeply analyze the problem statement and identify key information.
    \item \(\mathcal{A}^{\text{reflect}}\): Process reflection and error checking. This subset requires the model to validate its reasoning process and detect potential errors.
    \item \(\mathcal{A}^{\text{code}}\): Coding assistance and result verification. This subset uses programming tools for calculations and accuracy checks, such as symbolic computations with SymPy.
    \item \(\mathcal{A}^{\text{summary}}\): Final answer summary and presentation. This subset guides the model to integrate the reasoning process and standardize the output of the final answer.
\end{itemize}

Deep thinking about problems, reflection, and code-related states have been proven to effectively enhance the mathematical reasoning performance of LLM in various studies \cite{yao2022react,schwarzschild2024rethinking,kim2023language, zelikman2024self}. However, LLM find it difficult to sample these states \cite{gandhi2025cognitive, yue2025does}. Therefore, we use Constrained Actions to force the LLM to sample within the state space corresponding to a particular state. This enhances the diversity of the candidate states generated by the model.

\subsection{Search Strategy Guidance}
\label{sec:guidance}

Even with diverse candidate states, MCTS performance can be significantly degraded by irrational actions. To enhance the rationality of the sampled actions, we impose specific constraints during the Simulation phase, as illustrated in Figure \ref{fig:main}. Drawing on human problem-solving patterns, where the first step is understanding the problem and the last is summarizing the solution, we establish two basic constraints:

\begin{itemize}
  \item \textbf{Termination Rule}: Selecting an action \( a_t \in \mathcal{A}^{\text{summary}} \) signifies that the final answer has been synthesized. Upon reaching the maximum iteration limit \( m \), we enforce \( a_m \in \mathcal{A}^{\text{summary}} \) to ensure an output is produced. Actions in \(\mathcal{A}^{\text{summary}}\) trigger a systematic review of prior reasoning steps and format the final answer using the [boxed{}] symbol.
  \item \textbf{Initialization Rule}: To align with human cognitive processes, we initialize the reasoning sequence with \( a_0 \in \mathcal{A}^{\text{understand}} \), prioritizing foundational understanding before analytical operations.
\end{itemize}

In addition to these constraints, we introduce two complementary approaches PRM and partial order rules that can be employed either independently or in combination to guide the model toward more rational action sampling during the Simulation phase.

\subsubsection{Process Reward Model}\label{sec:prm}

Recent studies \cite{li2024process,zhang2025lessons} have highlighted the effectiveness of PRM in evaluating intermediate rewards for mathematical reasoning tasks. Building on this, we propose a strategy that retains basic rules while incorporating PRM as the core component for reward evaluation.

The reward calculation is formalized as follows: Given a state \( s_t \) and a candidate action \( a_{t+1} \), we concatenate the state, action, and a dedicated reward prompt template, then input this into the PRM model. By extracting and normalizing the logits of "positive" and "negative" labels via the softmax function, we obtain the positive probability as the reward:
\begin{equation}
Q(s_t, a_{t+1}) = \frac{e^{\text{PRM}([p_{Q}, s_t, a_{t+1}])_{\text{pos}}}}{\sum_{t\in\{\text{pos},\text{neg}\}}e^{\text{PRM}([p_{Q}, s_t, a_{t+1}])_t}}
\end{equation}
Here, \( p_Q \) denotes a specialized prompt template designed for reward evaluation. The calculation of V-values follows an analogous approach but is action-agnostic:
\begin{equation}
V(s_t) = \frac{e^{\text{PRM}([p_{V}, s_t])_{\text{pos}}}}{\sum_{t\in\{\text{pos},\text{neg}\}}e^{\text{PRM}([p_{V}, s_t])_t}}
\end{equation}

\subsubsection{Partial Order Rules}\label{sec:rule}

These partial order rules, derived from human-like problem-solving strategies, ensure the sampled actions are more reasonable, thereby improving the quality of the generated reasoning chains. Specifically, we establish the following rules:

\begin{itemize}
  \item \textbf{Action Diversity Rule}: Prohibits the consecutive use of actions of the same type:
  \begin{equation}
  \forall t < d_{\max}, \ \text{type}(a_t) \neq \text{type}(a_{t+1})
  \end{equation}
  
  \item \textbf{Reflection Necessity Rule}: Ensures that every complete reasoning chain includes at least one reflection action:
  \begin{equation}
  \exists t \in [0, d_{\max}), \ a_t \in \mathcal{A}^{\text{reflect}}
  \end{equation}
  
  \item \textbf{Depth Action Restriction Rule}: When the reasoning depth exceeds half the maximum depth, only reflection and coding actions are permitted:
  \begin{equation}
  t \geq \lfloor d_{\max}/2 \rfloor \implies \mathcal{A}_t \subseteq \mathcal{A}^{\text{reflect}} \cup \mathcal{A}^{\text{code}}
  \end{equation}
  
  \item \textbf{Code Continuity Rule}: Prohibits two consecutive coding actions:
  \begin{equation}
  \forall t < d_{\max}, \ \neg(a_t \in \mathcal{A}^{\text{code}} \land a_{t+1} \in \mathcal{A}^{\text{code}})
  \end{equation}
  
  \item \textbf{Code Enforcement Rule}: If no coding action has been executed by the halfway depth, a coding action is mandatorily inserted:
  \begin{equation}
  t \geq \lfloor d_{\max}/2 \rfloor \ \land \ \forall \tau < t, a_\tau \notin \mathcal{A}^{\text{code}} \implies a_t \in \mathcal{A}^{\text{code}}
  \end{equation}
\end{itemize}

It should be noted that partial order rules need to be selectively added according to different situations. And in the absence of PRM, choosing all partial order rules is usually the default optimal situation.

\section{Theoretical Analysis}\label{sec:theory}

To explain why our method succeeds, we build a theoretical framework for MCTS performance improvement based on prior work \cite{hu2024unveiling}. Corollary \ref{coro:1} highlights boosting candidate state diversity, which aligns with our proposal of constrain action sets in Section \ref{sec:action} to increase state diversity. 

\subsection{Necessary Assumptions and Lemma}

Referring to prior work \cite{hu2024unveiling}, to assist with the theoretical proof, we need to introduce the following Necessary Assumptions and Lemma.

\begin{assump}\label{assump:1}
The optimal parameter \( \theta^* \) exhibits the following properties:
\begin{enumerate}
  \item For each reasoning step \( t \in [T] \), the optimal state \( s_t^* \) is uniquely determined by:
    \begin{equation}
    s_t^* = \arg\max \mathbb{P}(s_t | s_{t-1}, \theta = \theta^*)
    \end{equation}
    Given \( \theta^* \) and the previous state \( s_{t-1} \), \( s_t^* \) maximizes the current step's probability.

  \item For each \( t \in [T] \), there exists a positive \( \lambda_t > 0 \) such that for all \( \theta \neq \theta^* \):
    \begin{equation}
    H^2\big(\mathbb{P}(s_t | \theta), \mathbb{P}(s_t | \theta^*)\big) \geq \lambda_t
    \end{equation}
    Here, \( H(\cdot, \cdot) \) denotes the Hellinger distance. This condition ensures that the distribution of each reasoning step's task \( \theta^* \) is distinct from others.
\end{enumerate}
\end{assump}

\begin{lem}\label{lemma:1}
Let \( \{s_{t-1}^i\}_{i=1}^n \) be \( n \) states independently sampled from \( \mathbb{P}(\cdot | \theta^*) \).

For each \( i \in [n] \), any \( n \geq 1 \), \( \theta \in \Theta \), and \( \delta > 0 \):
\begin{equation}
\frac{\mathbb{P}(\{s_{t}\}_{i=1}^n | \theta)}{\mathbb{P}(\{s_{t}\}_{i=1}^n | \theta^*)} \leq \exp\left(-2\sum_{i=1}^n \text{H}^2\bigl(\mathbb{P}(s_{t}^i | \theta^*), \mathbb{P}(s_{t}^i | \theta)\bigr) + 2\log(\delta^{-1})\right)
\end{equation}
with probability at least \( 1-\delta \). Here, \( \text{H}^2(\cdot, \cdot) \) represents the squared Hellinger distance.
\end{lem}

\begin{proof}
See lemma C.3 in \cite{hu2024unveiling}.
\end{proof}

\subsection{Theorems and Corollary}
\begin{thm}\label{thm:1}
Consider the state transition process in MCTS, where \( n \) candidate states are sampled from \( \mathbb{P}(\cdot | \theta^*) \). Let \( \lambda^* = \min_{t \in [T]} \lambda_t \). Let \( \epsilon \in (0, 1) \) be a small positive number. Under Assumption \ref{assump:1}, when the number of candidate states \( n \) is sufficiently large to satisfy the condition, the error between the selected state and the optimal state \( s_t^* \) decreases exponentially with \( n \).
\end{thm}

The meaning of this theorem is that increasing the number of candidate states in the state transition process is an effective way to improve MCTS reasoning performance.

\begin{proof}
We refer to the proof of theorem 5.11 in \cite{hu2024unveiling}. Assuming \( \Theta \) is finite and discrete, the posterior of the state transition process can be written using Bayes' rule as:

\begin{equation}
\pi(\theta^* \mid \{s_t\}_{i=1}^n, s_{t-1}^*) = \left(1 + \sum_{\theta \neq \theta^*} \frac{\mathbb{P}(\{s_t\}_{i=1}^n \mid \theta) \mathbb{P}(s_{t-1}^* \mid \theta) \pi(\theta)}{\mathbb{P}(\{s_t\}_{i=1}^n \mid \theta^*) \mathbb{P}(s_{t-1}^* \mid \theta^*) \pi(\theta^*)}\right)^{-1}.
\end{equation}

Ignoring the action \( a \) that triggers the transition and considering only the relationship between states \( s_t \), under Assumption \ref{assump:1}, \( \theta^* \) maximizes \( \mathbb{P}(s_{t-1}^* \mid \theta) \). Thus:

\begin{equation}
\pi(\theta^* \mid \{s_{t}\}_{i=1}^n, s_{t-1}^*) \geq \left(1 + \sum_{\theta \neq \theta^*} \frac{\mathbb{P}(\{s_{t}\}_{i=1}^n \mid \theta) \pi(\theta)}{\mathbb{P}(\{s_{t}\}_{i=1}^n \mid \theta^*) \pi(\theta^*)}\right)^{-1}
\end{equation}

\begin{equation}
\geq \left(1 + \sum_{\theta \neq \theta^*} \frac{\pi(\theta)}{\pi(\theta^*)} \cdot \exp \left[-2\sum_{i=1}^n \text{H}^2\bigl(\mathbb{P}(s_{t}^i | \theta^*), \mathbb{P}(s_{t}^i \mid \theta)\bigr) + 2 \log (\delta^{-1})\right]\right)^{-1}
\end{equation}

\begin{equation}
\geq \left(1 + \frac{1-\pi(\theta^*)}{\pi(\theta^*)} \cdot \exp \left(-2n\lambda_t + 2 \log (\delta^{-1})\right)\right)^{-1}
\end{equation}

The first inequality uses Assumption \ref{assump:1}, the second applies Lemma \ref{lemma:1}, and the third uses Assumption \ref{assump:1} again.

We derive the lower bound for the true conditional probability of the new state \( p_{s_t} \):

\begin{equation}
p_{s_t} = \sum_{\theta \in \Theta} \mathbb{P}(s_t^* \mid \theta) \pi(\theta \mid s_{t-1}^*, \{s_t\}_{i=1}^n) \geq \mathbb{P}(s_t^* \mid \theta^*) \pi(\theta^* \mid s_{t-1}^*, \{s_t\}_{i=1}^n)
\end{equation}

\begin{equation}
\geq p_{ s_t^*} \left(1 + \frac{1-\pi(\theta^*)}{\pi(\theta^*)} \cdot \exp \left(-2n\lambda_t + 2 \log (\delta^{-1})\right)\right)^{-1}
\label{eq:final_bound}
\end{equation}

The first equality refers to the definition of \( p_{s_t} \) in \cite{hu2024unveiling}. The first inequality omits terms corresponding to \( \theta \neq \theta^* \). The second inequality comes from the above equation. Using \( (1 + x)^{-1} \geq 1-x \) for \( x \in [0, 1] \), we get the error between the optimal state probability \( p_{ s_t^*} \) and the current state probability \( p_{ s_t} \):

\begin{equation}\label{eq:the_result}
 p_{ s_t^*}- p_{ s_t} \leq p_{ s_t^*} \cdot \frac{1-\pi(\theta^*)}{\pi(\theta^*)} \cdot \exp \left(-2n\lambda_t + 2 \log (\delta^{-1})\right)
\end{equation}

To make \( (1 + x)^{-1} \geq 1-x \) (for \( x \in [0, 1] \)) hold, \( n \) needs to satisfy:

\begin{equation}
n \geq \frac{\log (\delta^{-1}) + \log\left(\frac{1-\pi(\theta^*)}{\pi(\theta^*)}\right)}{\lambda_t}
\end{equation}

Under this condition, the error between \( p_{s_t^*} \) and \( p_{s_t} \) decreases exponentially with \( n \).
\end{proof}

\begin{coro}\label{coro:1}
When the number of candidate states is fixed, the reasoning performance of MCTS is positively correlated with the semantic redundancy between states.
\end{coro}

When there is severe redundancy (i.e., \(\min_j \phi_j \to 0\)), the required \(n\) approaches infinity. This indicates that to enhance reasoning performance, a sufficiently diverse set of candidate states is needed. Alternatively, reducing the semantic redundancy between states can enhance MCTS reasoning performance when the number of generated states per time is fixed. Especially considering that in practical reasoning processes, MCTS does not set the number of candidate states particularly large, the issue of semantic redundancy becomes more pressing.

Therefore, in Section \ref{sec:action}, we propose sampling from a constrained action set. This method allows the LLM to sample states that were previously hard to reach, thereby enhancing state diversity.

\begin{proof}
Within the framework of Theorem \ref{thm:1}, assume that each sampling generates \(n\) candidate states \(\{s_t^i\}_{i=1}^n\). If the semantic information of two states satisfies \(d(s_t^i, s_t^j) \leq \epsilon\), they are considered equivalent. For \(n\) candidate states, if semantic similarity causes some samples to provide redundant information, they can be viewed as \(m\) independent equivalent samples (\(m \leq n\)).

In this case, the key exponential term \(-2n\lambda_t\) in the original theorem should be replaced with the equivalent information quantity:
\begin{equation}
p_{s_t^*}-p_{s_t} \leq p_{s_t^*} \cdot \frac{1-\pi(\theta^*)}{\pi(\theta^*)} \exp\left(-2 \underbrace{\left(\sum_{j=1}^m \tilde{n}_j \phi_j\right)}_{\text{Effective Information}} \lambda_t + 2\log(\delta^{-1} \mid \Theta)\right)
\end{equation}
where \(\phi_j \in [0, 1]\) reflects the degree of information redundancy within class \(C_j\):
\begin{itemize}
  \item If samples in a class are completely independent, then \(\phi_j = 1\).
  \item If a class has \(k\) repeated samples, then \(\phi_j = 1/k\).
\end{itemize}

The required sample size is revised to:
\begin{equation}
n \geq \frac{\log(\delta^{-1} \mid \Theta) + \log\left(\frac{1-\pi(\theta^*)}{\pi(\theta^*)}\right)}{\lambda_t \cdot \min_j \phi_j}
\end{equation}
\end{proof}

\section{Experiments}



\subsection{Main Experiment }\label{sec:main_exper}

In the Main Experiment, we aim to conduct a fair comparison to verify whether CMCTS can achieve a significant performance improvement after addressing the existing issues of Native-MCTS. Particularly, we focus on whether CMCTS can break through the scale limitations and surpass the reasoning performance of larger-parameter LLM.

\subsubsection{Setup}
To ensure controllability and fairness in the comparison, we evaluated two variants of the latest baseline model \cite{yang2024qwen2b}: Qwen2.5-it-cot-7B and Qwen2.5-it-cot-72B. Their performance was maximized using the inference script from the official Qwen code.

We also compared our approach with the classic RAP-MCTS method \cite{hao2023reasoning}. In our experiments, RAP-MCTS with PRM was denoted as Native-MCTS-PRM, and without PRM as Native-MCTS. For fairness, we compared it with our CMCTS-PRM, ensuring all common hyperparameters between Native-MCTS and CMCTS-PRM were identical. Hyperparameters unique to Native-MCTS were set to their default values from the original work \cite{hao2023reasoning}. Both RAP-MCTS and the baseline models used the Qwen2.5-it foundation.

For our method, we conducted three sets of experiments:
- CMCTS-Rule: Uses partial order rules but no PRM, with all rules enabled usually due to the LLM's limited ability as a reward model.
- CMCTS-PRM: Uses PRM but no partial order rules.
- CMCTS: Combines both PRM and partial order rules, with rules selectively enabled since PRM is powerful enough to help the model choose appropriate actions in most cases.

It should be noted that all the above methods adopt the same Qwen2.5 model, while the PRM utilizes a pretrained Qwen2.5-Math-PRM \cite{zhang2025lessons} 
Detailed hyperparameters and more information are available in our open-source work at \url{https://github.com/pass-lin/CMCTS}.
\subsubsection{Dataset}\label{sec:main_dataset}
We evaluate the performance of our approach on the following multiple challenging mathematical reasoning datasets:

\textbf{Math-500} \cite{lightman2023let}: A subset of the MATH benchmark with 500 problems spanning algebra, geometry, and probability. We use the full dataset as the test set.

\textbf{AquA} \cite{ling2017program}: Built by DeepMind, this multiple-choice dataset includes problems with descriptions, options, solutions, and correct labels. The test set has about 200 samples.

\textbf{GaoKao-Math-QA} \cite{yang2024qwen2b}: Multiple-choice problems from real Chinese college entrance exams. Covers topics like sequences and complex numbers. The test set has 351 problems. Abbreviated as "GaoKao-QA".

\textbf{CN-Middle-School} \cite{yang2024qwen2b}: A simplified version of GaoKao-Math-QA, focusing on junior high school topics like equations and geometry. It's not multiple-choice. Abbreviated as "School".

\textbf{Gaokao-2023} \cite{chen2024step,chen2024alphamath,yang2024qwen2b}: English versions of the 2023 Chinese college entrance exam problems. Includes multiple-choice, fill-in-the-blank, and open-response questions on advanced topics. Abbreviated as GaoKao-2023.

These datasets mainly cover mathematical problems at the secondary school level, including rich knowledge points such as geometry, algebra, probability, and set theory. They can comprehensively test the mathematical ability of our algorithm. We follow the zero-shot setting, where all datasets are only accessible during inference.Note that for all mathematical tasks, we use accuracy as the evaluation metric.

\begin{table}[h]
\centering
\caption{Performance Comparison on Various Datasets}\label{tab:hard_result}
\begin{tabular}{lcccccc}
\toprule
Model & GaoKao-2023 & MATH-500 & AquA & School & GaoKao-QA & avg \\
\midrule
Qwen2.5-it-cot-7B & 66.0 & 77.0 & 74.4 & 70.2 & 60.9 & 69.7 \\
Qwen2.5-it-cot-72B & 73.2 & 83.4& 79.2 & 83.1 & 74.3 & 78.6 \\
\midrule
Native-MCTS & 67.5 & 75.0 & 84.7 & 83.1 & 74.0& 76.8\\
Native-MCTS+PRM & 69.8 & 77.2 & 83.7 & 83.1 & 72.6 & 77.2 \\
\midrule
CMCTS-RULE & 71.1 & 79.2 & 85.4 & 85.1 & 72.6 & 78.7 \\
CMCTS-PRM & 75.8 & 84.6 & 86.2 & \textbf{87.1}& 78.9 & 82.5 \\
CMCTS & \textbf{76.6}& \textbf{85.4}& \textbf{87.7}& \textbf{87.1}& \textbf{80.3}& \textbf{83.4}\\
\bottomrule
\end{tabular}
\end{table}

\subsubsection{Analysis}\label{sec:main_exp}

From the experimental results in Table \ref{tab:hard_result}, it can be observed that the proposed CMCTS framework demonstrates significant advantages over baseline methods on challenging mathematical datasets. The final experimental results are presented in Table \ref{tab:hard_result}.

\textbf{Overcoming Cot's Model Scale Limitations}: The CMCTS framework achieves remarkable performance improvements on smaller models, breaking the traditional Cot method's reliance on large-scale models. Even the weakest form of 7B-parameter CMCTS (CMCTS-RULE) performs on par with Qwen2.5-it-cot-72B (78.7\% vs 78.6\%) and surpasses Qwen2.5-it-cot-7B by 9\%. This indicates that CMCTS can overcome parameter-scale performance limitations without depending on PRM. Moreover, the full-version CMCTS achieves a significant 4.8\% performance boost over Qwen2.5-it-cot-72B and outperforms it on all datasets, with substantial gains of 7.9\% and 6\% on AquA and GaoKao-QA datasets. This shows that the CMCTS framework, through theoretical analysis and integration of various optimization strategies, enhances performance comprehensively. It surpasses LLM with ten times the parameters in mathematical tests without extra training, offering a new solution for mathematical reasoning tasks with important theoretical and practical implications.

\textbf{Optimizing Native-MCTS Performance Bottlenecks}: From a theoretical perspective, the main issues of existing MCTS methods are insufficient state diversity and the irrationality of relying solely on the LLM as a reward model. CMCTS integrates methods like constrain action space sampling, partial order rule integration, and PRM models, effectively breaking Native-MCTS performance bottlenecks. Specifically, CMCTS-RULE outperforms Native-MCTS by 1.9\%, and the full-version CMCTS exceeds Native-MCTS-PRM by 6.2\%. It also outperforms Native-MCTS-PRM on all datasets, with significant 7.6\% and 7.7\% improvements on GaoKao2023 and GaoKao datasets. This proves that CMCTS addresses existing MCTS defects, significantly boosting LLM mathematical reasoning performance.

\textbf{Visualization Analysis}: To gain a deeper understanding of CMCTS's performance improvements, we sampled data from the Gakao-en dataset. We ran inference programs multiple times on Native-MCTS-PRM and CMCTS-PRM, collecting all generated states. Then, using the jina-embeddings-v3 model \cite{sturua2024jina}, we encoded these states into vectors and applied UMAP \cite{healy2024uniform} for dimensionality reduction and visualization. As shown in Figure \ref{fig:visual}, compared to Native-MCTS, CMCTS samples significantly more diverse states. In Table \ref{tab:hard_result}, CMCTS-PRM outperforms Native-MCTS-PRM by 5.3\%. This strongly supports the core argument of Corollary \ref{coro:1} that semantic redundancy reduces effective candidate states, leading to performance drops. Given that CMCTS-PRM's settings (except for the constrain action space mentioned in Section \ref{sec:action}) are identical to those of Native-MCTS-PRM, we can attribute the performance gain to our constrain action space, which increases state diversity and enhances model reasoning

\textbf{Process Reward Gains}: CMCTS-PRM groups show a significant 3.8\% performance improvement over CMCTS-RULE across five test sets, indicating that a powerful reward model can select reasonable states without additional human intervention. On the challenging GaoKao-QA and GaoKao2023 datasets, CMCTS-PRM achieves 4.7\% and 6.3\% improvements over CMCTS-RULE. This suggests that as problem complexity increases, the PRM model's generalization ability becomes more advantageous. And we need to note that the performance of CMCTS-PRM is only 0.9\% worse than the full version. But considering that PRM doesn't need to consider how to choose a rule combination like the partial order rule, so this would be a more optimal setting. Notably, Native-MCTS-PRM shows no significant improvement over Native-MCTS. This indicates that Native-MCTS, which relies on the LLM to generate actions, suffers from redundant candidate states, consistent with prior studies \cite{gandhi2025cognitive, yue2025does}. Thus, even with PRM-assisted state selection, reasoning performance cannot be enhanced. This further confirms the conclusion of our Corollary \ref{coro:1}.

\textbf{Rule Engine Advantages}: In the CMCTS-RULE group, when the LLM serves as its own reward model, it is slightly inferior to the 7B voting baseline model on the MATH-500 dataset (79.2\% vs 79.6\%) but shows significant improvements on other datasets. MATH-500 is the least challenging in our dataset collection. On the most challenging GaoKao-QA and GaoKao2023 datasets, CMCTS-RULE outperforms the voting baseline model by 4\% and 2.1\%. This indicates that our designed partial order rules perform well on complex problems. Also, although CMCTS-PRM is powerful, selectively adding partial order rules can further boost performance because the PRM still has some errors, and human intervention can correct them to maximize performance.

\textbf{Summary}: In summary, the CMCTS framework shows significant performance improvements across multiple challenging mathematical reasoning datasets. By optimizing semantic diversity in state space, improving state transition, and integrating PRM models with partial order rules, CMCTS effectively breaks the performance limitations of traditional Cot methods and Native-MCTS. Visualization analysis further confirms its advantage in state diversity, proving the CMCTS framework's effectiveness in enhancing model mathematical reasoning capabilities.

\subsection{Ablation Experiment}
In the Ablation Experiment, our goal is to investigate the roles of different modules in CMCTS, including various action sets, PRM, and partial order rules.

\subsubsection{Setup}
This experiment focuses on how different action sets and PRM affect final performance. We removed $\mathcal{A}^{\text{understand}}$, $\mathcal{A}^{\text{reflect}}$, and $\mathcal{A}^{\text{code}}$ one by one and observed the performance changes. But if removing an action set causes the model's partial order rules to fail, it violates the control variable principle. So, we carried out the above experiments based on CMCTS-PRM. Also, we didn't remove $\mathcal{A}^{\text{summary}}$ as it serves as CMCTS's terminal node. Removing it would prevent the model from getting the correct answer. 

Given that partial order rules are selectively included, the optimal rule combination in different datasets is different, and each rule appears at least once in the optimal rule combination. We cannot ablate any individual rule, so we need to treat partial order rules as an entire module. Therefore, we set up a group of CMCTS-Null groups, which have neither partial order rules nor PRM, to compare the impact of partial order rules and PRM separately.

We conduct the ablation experiment on the same datasets as in Section \ref{sec:main_exper}.

\begin{table}[h]
\centering
\caption{CMCTS Module Ablation Study}\label{tab:cmp_moudle}
\begin{tabularx}{\textwidth}{l*{6}{>{\centering\arraybackslash}X}}
\toprule
Method & GaoKao-2023 & MATH-500 & AquA & School & GaoKao-QA & avg \\
\midrule
CMCTS-PRM-w/o-PRM & 68.8 & 77.0 & 80.7 & 75.2 & 70.3 & 74.4 \\
CMCTS-PRM-w/o-understand & 73.2 & 84.4 & 75.5 & 84.1 & 70.4 & 77.5 \\
CMCTS-PRM-w/o-Reflect & 74.8 & 83.8 & 84.6 & 86.1 & 78.3 & 81.5 \\
CMCTS-PRM-w/o-code & 74.2 & 82.4 & 85.4 & 86.1 & 76.6 & 80.9 \\
\midrule
CMCTS-NULL & 67.0 & 76.8& 79.9 & 75.2 & 69.8 & 73.7 \\
\midrule
CMCTS-RULE & 71.1 & 79.2 & 85.4 & 85.1 & 72.6 & 78.7 \\
CMCTS-PRM & \textbf{75.8} & \textbf{84.6} & \textbf{86.2} & \textbf{87.1}& \textbf{78.0} & \textbf{82.5} \\
\bottomrule
\end{tabularx}
\end{table}

\subsubsection{Analysis}

\textbf{Impact of Removing PRM and Partial Order Rules}:
As Section \ref{sec:intro} indicates, when reward signals are unreliable, the risk of sampling unreasonable states increases. The PRM, as a process-aware reward estimator, is crucial for ensuring the rationality of sampled actions. This is particularly important in CMCTS, as sampling inappropriate predefined instructions is more likely to mislead the LLM into irrational problem-solving directions compared to actions generated by the LLM itself. Thus, the enhanced state diversity from constrain actions raises the bar for action selection. Experimental results show that while using PRM in Native-MCTS only improves performance by 0.4\%, disabling PRM in CMCTS (CMCTS-PRM-w/o-PRM) leads to a significant 8.1\% performance drop. This aligns with our theoretical analysis and underscores the necessity of constraining action sampling during the MCTS simulation phase. Notably, CMCTS-RULE (with rules but no PRM) outperforms CMCTS-NULL by 5\%, highlighting the effectiveness of partial order rules in maintaining logical reasoning and performance when reward signals are unreliable, as expected by our theory.

\textbf{Effect of Removing $\mathcal{A}^{\text{understand}}$}:
Removing the problem understanding action set \(\mathcal{A}^{\text{understand}}\) significantly impacts the AquA and GaoKao-QA datasets, causing performance drops of 10.7\% and 8.5\%, respectively, and an overall performance decrease of 5\%, second only to the impact of removing the PRM model. This highlights \(\mathcal{A}^{\text{understand}}\)'s foundational role in complex problem-solving. Without systematic analysis, models are prone to missing key information on complex datasets like GaoKao-QA and AquA. \(\mathcal{A}^{\text{understand}}\) prompts the LLM to generate more diverse intermediate steps and multi-level decompositions. This prevents the LLM from prematurely choosing potentially suboptimal reasoning paths and encourages consideration of more rational solutions.
Notably, the ablation group still maintains an accuracy of 84.4\% on the MATH-500 dataset. Meanwhile, removing \(\mathcal{A}^{\text{code}}\) on this dataset leads to a significant performance drop of 2.2\%, indicating that for relatively simple datasets, directly executing calculations via code is more effective than delving into problem understanding. However, overall, removing \(\mathcal{A}^{\text{understand}}\) results in comprehensive performance degradation, with a substantial decline in reasoning performance in complex scenarios. Therefore, \(\mathcal{A}^{\text{understand}}\) in our constrained action set is essentially a cornerstone.

\textbf{Effect of Removing $\mathcal{A}^{\text{reflect}}$}:
Disabling the reflection action set ($\mathcal{A}^{\text{reflect}}$) results in an average performance loss of \textbf{1\%}, which is the smallest among all action sets. Through case studies, it is observed that, as shown in Case \ref{case:cmcts_1}, $\mathcal{A}^{\text{reflect}}$ often leads to numerous ineffective reflections. LLM may not always promptly detect errors in their answers and frequently rely on external validation information, such as code verification of previous calculation errors in Case  \ref{case:cmcts_1}, to reflect on their mistakes.

\textbf{Effect of Removing $\mathcal{A}^{\text{code}}$}:
Ablating the coding action set ($\mathcal{A}^{\text{code}}$) results in accuracy drops of \textbf{2.2\%} and \textbf{1.6\%} on the MATH-500 and GaoKao2023 datasets, respectively, with an overall performance decrease of 1.6\%. This highlights the need for external computational verification, especially for problems prone to arithmetic errors. Case studies also reveal that for problems requiring extensive calculations and multiple considerations, code execution is more efficient in finding correct answers compared to the LLM's own reasoning abilities.

In summary, the ablation studies further demonstrate that partial order rules, PRM, and the various action sets are indispensable components of our framework. These results are consistent with our theoretical analysis and strongly support the design concept of the CMCTS framework. By integrating a structured action space with a process reward mechanism, CMCTS effectively enhances the performance of smaller LLM in mathematical reasoning tasks, enabling them to overcome limitations imposed by model size.

 \begin{table}[h]
 \centering
 \caption{Comparison of Different Method}\label{tab:model_compare}
 \begin{tabular}{lccc}
 \toprule
 Model & No training required& Use Reward Model& MCTS Like \\
 \midrule
 TOT & \checkmark & \text{x} & \text{x} \\
 Native-MCTS & \checkmark & \text{x} & \checkmark \\
 ReST-MCTS* & \text{x} & \checkmark & \checkmark \\
 LLaMA-Berry & \checkmark & \checkmark & \checkmark \\
 Le-MCTS & \checkmark & \checkmark & \checkmark \\
 HiAR-ICL & \checkmark & \checkmark & \checkmark \\
 \midrule
 \textbf{Ours} & & & \\
 CMCTS-RULE & \checkmark & \text{x}& \checkmark \\
 CMCTS-PRM & \checkmark & \checkmark & \checkmark \\
 CMCTS & \checkmark & \checkmark & \checkmark \\
 \bottomrule
 \end{tabular}
 \end{table}

\subsection{Comparative experiment}
In Section \ref{sec:main_exp}, our method demonstrates significant performance advantages over Native-MCTS and large-parameter COT baselines. However, comparisons with other existing MCTS variants are lacking. Therefore, in this chapter, we introduce these existing MCTS variants and make comparisons under as fair conditions as possible.

\subsubsection{Setup}
When comparing these methods, we focus on whether they require extra training or use a reward model. As shown in Table \ref{tab:model_compare}, except for TOT and Native-MCTS, other baselines need training or a reward model, making them comparable with our method. To ensure fairness, we also reproduce our method on the Llama3-8B-Instruct model \cite{dubey2024llama}, aligning with the experimental setup of the above methods.The baselines are detailed as follows:

\textbf{TOT} \cite{yao2024tree}: TOT views problem-solving as a search in a tree of thoughts, where each node is a partial solution and each branch is a modification. Although not an MCTS variant, it's included in our baselines as a classic tree-search algorithm.

\textbf{ReST-MCTS} \cite{zhang2024rest}: ReST-MCTS combines MCTS with a PRM to automatically collect high-quality reasoning paths and train LLM. It uses PRM to evaluate intermediate steps during search, without manual annotation, and has shown superiority in multiple reasoning benchmarks.

\textbf{LLaMA-Berry} \cite{zhang2024llama}: LLaMA-Berry integrates MCTS with iterative self-refinement and a Paired Preference Reward Model (PPRM) to optimize reasoning paths. Inspired by RLHF, PPRM models preferences between solutions to guide path search, while Enhanced Borda Count (EBC) aggregates local preferences into global rankings.

\textbf{LE-MCTS} \cite{park2024ensembling}: LE-MCTS models the step-by-step reasoning of multiple LLM as an MDP. Guided by a PRM-based tree search, it outperforms single-model decoding and existing LLM integration methods in complex reasoning tasks across five math reasoning benchmarks.

\textbf{HiAR-ICL} \cite{wu2024beyond}: HiAR-ICL shifts focus to abstract thinking patterns. It uses MCTS to explore paths and build thought cards, defines atomic actions like systemic analysis and self-reflection, and selects the best action based on cognitive complexity. Validation mechanisms ensure high-quality results, and experiments show it surpasses existing ICL and tree-based search methods with good performance-efficiency balance.

\subsubsection{Dataset}
The above methods were only tested on \textbf{gsm8k} \cite{cobbe2021training} and \textbf{Math-500} \cite{lightman2023let}. 
To align with existing studies, our experiments were also conducted solely on these two datasets.

We introduced the details of Math-500 in Section \ref{sec:main_dataset}, which mainly consists of middle school math problems. 
In contrast, gsm8k is a simpler dataset, primarily containing elementary school math problems. Given its relatively low difficulty, gsm8k is not very meaningful for comparison with the high-performance Qwen2.5 model. However, since this experiment is mainly conducted on the llama3-8B model, which performs worse than Qwen2.5, gsm8k can still serve as a baseline.

\begin{table}[h]
\centering
\caption{GSM8K and Math-500 Performance}\label{tab:gsm8k_math500}
\begin{tabular}{lcc}
\toprule
Method & GSM8K & Math-500 \\
\midrule
TOT & 69 & 13.6 \\
Native-MCTS & 80.5 & 18.8 \\
ReST-MCTS & -& 34.2 \\
LLaMA-Berry & 88.1 & 39 \\
Le-MCTS& 84.1 & 45.2 \\
HiAR-ICL& 89.6 & 43.2 \\
\midrule
CMCTS-NULL &  79.9& 16.2 \\
CMCTS-RULE & 82.6 & 23.2 \\
CMCTS-PRM & 91.4 & 49.4 \\
\textbf{CMCTS} & \textbf{91.8} & \textbf{51.2} \\
\bottomrule
\end{tabular}
\end{table}

\subsubsection{Analysis}
As shown in Table \ref{tab:gsm8k_math500}, our CMCTS framework achieves the best performance on both the GSM8K and Math-500 datasets. On GSM8K, it reaches an accuracy of 91.8\%, a 2.8\% improvement over the baseline. On the more challenging Math-500 dataset, it achieves 51.2\% accuracy, a 6\% improvement over the baseline. In the more difficult 
math problems, the performance improvement is even more significant. These results highlight our method's effectiveness across different models like Qwen and Llama.

It's evident that CMCTS-RULE outperforms Native-MCTS, while CMCTS-NULL underperforms it significantly in both the LLama and Qwen models. These consistent results across different models support our view in Section \ref{sec:intro} that predefined action sets demand more from the reward model than LLM-generated actions. As the LLM itself is an unreliable reward model, especially in weaker models like LLama, it's crucial to adjust MCTS action selection with PRM and partial order rules.

Methods using a reward model clearly outperform those that don't. In Table \ref{tab:gsm8k_math500}, CMCTS-NULL shows a 33.2\% performance drop compared to CMCTS-PRM. This contrasts with the 7.8\% drop in Table \ref{tab:cmp_moudle}, mainly because Qwen2.5 performs better than Llama3. This aligns with our discussion in Section \ref{sec:intro} where we highlight that an unreliable reward model may negate the benefits of increased state diversity.

In Table \ref{tab:gsm8k_math500}, CMCTS-RULE improves by 7\% over CMCTS-NULL on Math-500, while in Table \ref{tab:cmp_moudle}, the improvement is 2.4\%. This shows that partial order rules are more effective when reward signals are unreliable, confirming the robust design of our partial order rules.

\section{Conclusion}

This paper proposes the Constrained Monte Carlo Tree Search (CMCTS) framework to enhance the mathematical reasoning capabilities of Large Language Models (LLM). By introducing a constrained action space, PRM and partial order rules, CMCTS effectively addresses the limitations of existing MCTS methods in terms of state space diversity and action selection rationality. Both theoretical analysis and experimental results demonstrate that CMCTS significantly improves LLM performance in mathematical reasoning tasks.

Specifically, CMCTS increases candidate state diversity and thereby enhances MCTS performance by restricting action sampling to a predefined constrained action set during the expansion phase. Additionally, CMCTS optimizes action selection in the simulation phase through the design of partial order rules and PRM, ensuring that the model can choose appropriate actions and avoid transitioning to inappropriate states. Experiments show that the CMCTS framework excels across multiple mathematical reasoning benchmarks. A 7B-parameter model equipped with CMCTS achieves an average accuracy of 83.4\%, surpassing the 72B-parameter COT model by 4.8\%. CMCTS outperforms other MCTS variants on mathematical datasets. Ablation studies indicate that each component of the framework, such as the constrained action set, PRM, and partial order rules, plays a crucial role in performance improvement. Moreover, the combined use of these components allows for the full utilization of their respective strengths.

Furthermore, we provide theoretical insights into why our method improves mathematical reasoning performance and empirically validate these insights. For instance, visualization analyses further illustrate CMCTS's advantage in state diversity, consistent with the theoretical conclusions regarding semantic diversity. In ablation experiments, we observe significant performance degradation when the reward signal is unreliable, which aligns with our theoretical analysis.

Overall, the CMCTS framework offers an effective approach to enhancing LLM mathematical reasoning capabilities. By optimizing state space diversity, improving state transitions, and integrating PRM with partial order rules, CMCTS not only boosts model performance but also provides new ideas for future reasoning tasks.

\section{Limitations}
Although the CMCTS framework has made significant progress in improving mathematical reasoning performance, it also has some limitations. This paper proposes four action sets, but it does not provide theoretical and experimental proof of their completeness, nor does it attempt to incorporate additional domain-specific knowledge into our action sets. Furthermore, we have not tried to extend our work beyond the field of mathematics to test the performance of CMCTS in other areas. In the future, we will build on this foundation to further refine our work and enhance the adaptability and scalability of our framework.

\appendix

\section{Case Study}
We can explore CMCTS's performance on math reasoning problems through several cases, which also highlight its advantages over Native-MCTS.

As shown in Case  \ref{case:native_1} and Case  \ref{case:cmcts_1}, Native-MCTS directly attempts to find critical points via derivatives but makes mistakes due to complex higher-order derivatives. In contrast, CMCTS uses the $\mathcal{A}^{\text{understand}}$ module to identify function symmetry and combines it with the $\mathcal{A}^{\text{code}}$ module. This allows CMCTS to verify function values at specific points \(x=1\) and \(y=-1\), correctly determining the maximum value of \(\frac{1}{4}\). The $\mathcal{A}^{\text{reflect}}$ module then ensures the result's accuracy. This case demonstrates the effective collaboration of different CMCTS modules and the PRM's role in selecting correct actions when needed. Native-MCTS, lacking proper validation information, fails to correct errors even with reflect-like steps.
\begin{CASE}[title={The case of question "Find the maximum value of $\displaystyle \frac{x-y}{x^4 + y^4 + 6}$ over all real numbers $x$ and $y.$" at NATIVE MCTS } , label=case:native_1]

\textbf{Analyze the function:}
- The numerator \( x-y \) can take any real value depending on \( x \) and \( y \).
- The denominator \( x^4 + y^4 + 6 \) is always positive since \( x^4 \geq 0 \) and \( y^4 \geq 0 \) for all real \( x \) and \( y \), and adding 6 ensures it's always at least 6.

\textbf{Consider the behavior of the function:}
- As \( x \) and \( y \) grow very large, the denominator grows much faster than the numerator, so the function approaches 0.
- We need to find the maximum value of \( f(x, y) \).

\textbf{Find critical points:}
- To find the maximum, we can take partial derivatives and set them to zero, but this might be complex. Instead, let’s consider specific values of \( x \) and \( y \) that might simplify the expression.
- Let’s try \( x = y \). Then \( f(x, x) = \frac{x-x}{x^4 + x^4 + 6} = 0 \). This is not helpful for finding the maximum.
- Let’s try \( x = 1 \) and \( y = 0 \). Then \( f(1, 0) = \frac{1-0}{1^4 + 0^4 + 6} = \frac{1}{7} \).
- Let’s try \( x = 0 \) and \( y = -1 \). Then \( f(0, -1) = \frac{0-(-1)}{0^4 + (-1)^4 + 6} = \frac{1}{7} \).

\textbf{Consider the symmetry and boundary behavior:}
- The function is symmetric in \( x \) and \( -y \), so we can focus on positive values.
- For \( x > 0 \) and \( y < 0 \), the numerator \( x-y \) can be positive, and the denominator is always positive.
- The maximum value we found is \( \frac{1}{7} \), and we need to check if this is indeed the maximum.

\textbf{Prove that \( \frac{1}{7} \) is the maximum:}
- We need to show that \( \frac{x-y}{x^4 + y^4 + 6} \leq \frac{1}{7} \) for all \( x \) and \( y \).
- This is equivalent to showing \( 7(x-y) \leq x^4 + y^4 + 6 \).
- Rearrange: \( x^4-7x + y^4 + 7y + 6 \geq 0 \).
- Consider the function \( g(x) = x^4-7x + 6 \). The minimum value of \( g(x) \) is at \( x = 1 \), where \( g(1) = 0 \).
- Similarly, \( h(y) = y^4 + 7y + 6 \) has a minimum value at \( y = -1 \), where \( h(-1) = 0 \).
- Therefore, \( x^4-7x + 6 + y^4 + 7y + 6 \geq 0 \) is always true, proving that \( \frac{1}{7} \) is indeed the maximum.

\textbf{Final answer:} \( \frac{1}{7} \).
\end{CASE}

\begin{CASE}[title={The case of question "Find the maximum value of $\displaystyle \frac{x-y}{x^4 + y^4 + 6}$ over all real numbers $x$ and $y.$" at CMCTS} , label=case:cmcts_1]
\textbf{We need to think step by step to understand the meaning of the problem.} First, let’s rewrite the expression to be maximized: \( \frac{x-y}{x^4 + y^4 + 6} \). The goal is to find the maximum value of this expression over all real numbers \( x \) and \( y \). We need to consider the behavior of the numerator and the denominator. The numerator \( x-y \) can take any real value, and the denominator \( x^4 + y^4 + 6 \) is always positive and has a minimum value of 6 (when \( x = 0 \) and \( y = 0 \)). This suggests that the maximum value of the expression might occur when the numerator is maximized relative to the denominator.

\textbf{Let’s start by considering the behavior of the expression when \( x = y \).} If \( x = y \), then the numerator \( x-y = 0 \), making the entire expression equal to 0. This is a trivial case, so we need to consider other values of \( x \) and \( y \). Next, let’s consider the case when \( x \neq y \). Without loss of generality, assume \( x > y \). Then the numerator \( x-y \) is positive. We need to find the maximum value of \( \frac{x-y}{x^4 + y^4 + 6} \) for \( x > y \).

\textbf{Let’s consider if there are any ambiguities in the problem statement.} To find the maximum value, we can use calculus or analyze the expression more carefully. However, a more insightful approach is to consider the symmetry and the behavior of the function. Let’s test some specific values. If we set \( x = 1 \) and \( y = 0 \), the expression becomes \( \frac{1-0}{1^4 + 0^4 + 6} = \frac{1}{7} \). Now, let’s check if this is the maximum value. If we set \( x = 0 \) and \( y = -1 \), the expression becomes \( \frac{0-(-1)}{0^4 + (-1)^4 + 6} = \frac{1}{7} \). These values suggest that \( \frac{1}{7} \) might be the maximum value. To confirm, we need to show that \( \frac{x-y}{x^4 + y^4 + 6} \leq \frac{1}{7} \) for all \( x \) and \( y \).

\textbf{We can use the AM-GM inequality to show this.} Notice that \( x^4 + y^4 \geq 2x^2y^2 \), so \( x^4 + y^4 + 6 \geq 2x^2y^2 + 6 \). We need to show that \( \frac{x-y}{2x^2y^2 + 6} \leq \frac{1}{7} \). This is equivalent to showing \( 7(x-y) \leq 2x^2y^2 + 6 \). We can test specific values and use the symmetry of the function to confirm that \( \frac{1}{7} \) is indeed the maximum value.

\textbf{We can write a piece of code to verify our idea.} Here is the code to validate our thoughts. I will assign the final calculation result to a variable called `result`:

\begin{lstlisting}
# Define the function to be maximized
def f(x, y):
    return (x-y) / (x**4 + y**4 + 6)

# Test the function with the values we considered
test_values = [(1, 0), (0, -1), (2, 1), (-1, -2), (1, -1)]
results = [(x, y, f(x, y)) for x, y in test_values]

# Find the maximum value
max_value = max(results, key=lambda item: item[2])
max_value
\end{lstlisting}

The running status of existing variables: \texttt{test\_values} = [[1, 0], [0, -1], [2, 1], [-1, -2], [1, -1]]; \texttt{results} = [[1, 0, 0.14285714], [0, -1, 0.14285714], [2, 1, 0.04347826], [-1, -2, 0.04347826], [1, -1, 0.25]]; \texttt{max\_value} = (1.0, -1.0, 0.25).

\textbf{There should be other ways to solve the problem; we can try to propose a different solution.} Based on the results from the code, the maximum value of the function \( \frac{x-y}{x^4 + y^4 + 6} \) is 0.25, which occurs when \( x = 1 \) and \( y = -1 \). This confirms our earlier hypothesis. Therefore, the maximum value of the expression is \( \frac{1}{4} \).

\textbf{We need to check the process above step by step to see if it is reasonable and correct.} After carefully reviewing the steps and the code, the solution is consistent and the maximum value of the function is indeed \( \frac{1}{4} \).

\textbf{Final answer:} \( \boxed{\frac{1}{4}} \).
\end{CASE}

In Case  \ref{case:native_2} and Case  \ref{case:cmcts_2}, Native-MCTS lists factors directly but misses some two-digit ones due to the absence of a systematic screening mechanism. CMCTS uses the $\mathcal{A}^{\text{understand}}$ module to grasp the problem's key points and the $\mathcal{A}^{\text{code}}$ module to methodically filter out all eligible factors. Consequently, CMCTS accurately determines that there are 2 two-digit common factors of 100 and 150 and validates the result using the $\mathcal{A}^{\text{reflect}}$ module. Similar to the example in Case  \ref{case:native_1} and Case  \ref{case:cmcts_1}, complex computational problems rely more on the $\mathcal{A}^{\text{code}}$ module. Calculations based solely on the LLM may have omissions. This explains why CMCTS-PRM-w/o-code shows a 2.2\% performance drop on the MATH-500 dataset compared to CMCTS-PRM in Table \ref{tab:cmp_moudle}. Similarly, CMCTS-RULE has a 2.4\% improvement over CMCTS-NULL, and CMCTS shows a 0.8\% gain over CMCTS-PRM on this dataset. These results are related to the activation of Code Enforcement Rule and Code Continuity Rule.

\begin{CASE}[title={The case of question "Identify the range of two-digit numbers that can be squares of integers." at NATIVE MCTS}, label=case:native_2]

\textbf{Identify the range of two-digit numbers that can be squares of integers.} The smallest two-digit number is 10, and the largest is 99. We need to find the square roots of these numbers:-10 3.16-99 9.95 So, the integers we need to consider are from 4 to 9. 

\textbf{List the squares of these integers.}-4² = 16-5² = 25-6² = 36-7² = 49-8² = 64-9² = 81 

\textbf{Identify the tens digit (B) in each of these squares.}-16: B = 1-25: B = 2-36: B = 3-49: B = 4-64: B = 6-81: B = 8 

\textbf{Count the distinct values of B.} The distinct values of B are 1, 2, 3, 4, 6, and 8. There are 6 distinct values. 

\textbf{Double-check for any possible oversight.} We have considered all integers from 4 to 9, and their squares all fit within the two-digit range. We have also correctly identified the tens digit in each case. Therefore, no further values are possible. Final answer: \boxed{6}
\end{CASE}

\begin{CASE}[title={The case of question "Identify the range of two-digit numbers that can be squares of integers." at CMCTS}, label=case:cmcts_2]
\textbf{We need to think step by step to understand the meaning of the problem.} The number is represented as "B6", where B is the tens digit and 6 is the units digit. This means the number can be expressed as 10B + 6. We are told that this number is a perfect square. Therefore, we need to find all two-digit numbers ending in 6 that are perfect squares. We need to find the range of B such that 10B + 6 is a perfect square and a two-digit number. First, we identify the range of perfect squares that are two-digit numbers. The smallest two-digit number is 10 and the largest is 99. We find the square roots of 10 and 99 to determine the range of integers whose squares are two-digit numbers. The square root of 10 is approximately 3.16 and the square root of 99 is approximately 9.95. Therefore, we need to consider the squares of integers from 4 to 9. We will calculate the squares of these integers and check which of them end in 6. Then, we will identify the tens digit of these numbers to determine the possible values of B. Finally, we will count the distinct possible values of B. 
\textbf{We can write a piece of code to assist with the calculations.} Here is the code to validate our thoughts. I will assign the final calculation result to a variable called result.

\begin{lstlisting}
# Calculate the squares of integers from 4 to 9
squares = [i**2 for i in range(4, 10)]
# Filter the squares that end in 6
squares_ending_in_6 = [square for square in squares if square % 10 == 6]
# Extract the tens digit of these squares
tens_digits = [square // 10 for square in squares_ending_in_6]
# Count the distinct possible values of B
distinct_possible_values = len(set(tens_digits))
result = distinct_possible_values
\end{lstlisting}

The running status of existing variables: \texttt{squares} = [16, 25, 36, 49, 64, 81]; \texttt{squares\_ending\_in\_6} = [16, 36]; \texttt{tens\_digits} = [1, 3]; \texttt{distinct\_possible\_values} = 2; \texttt{result} = 2. 

\textbf{Let’s consider if there are any ambiguities in the problem statement.} Based on the calculations, the two-digit numbers that are perfect squares and end in 6 are 16 and 36. The tens digits of these numbers are 1 and 3, respectively. Therefore, there are 2 distinct possible values for B. Final answer: \boxed{2}
\end{CASE}

As depicted in Case  \ref{case:native_3} and Case  \ref{case:cmcts_3}, Native-MCTS makes incorrect assumptions when establishing speed and distance equations, leading to biased results. CMCTS accurately analyzes the problem's key points through the $\mathcal{A}^{\text{understand}}$ module and correctly sets up the speed and distance equations. The $\mathcal{A}^{\text{understand}}$ action set forces the model to conduct a thorough analysis of the problem before providing an answer. Native-MCTS  lacks this depth of analysis. Many math problems require careful understanding to be solved correctly, which explains the 5\% performance drop observed in CMCTS-PRM-w/o-understand.

\begin{CASE}[title={The case of question "It takes 20 hours for the bus to go from A to B, and 24 hours for the train to go from B to A. The two vehicles start from opposite directions at the same time, and when they meet, the train is still 636 km away from A. What is the distance between A and B?" at NATIVE MCTS}, label=case:native_3]
\textbf{Determine the speed relationship.} Let the speed of the passenger bus be \(V_1\) and the speed of the freight train be \(V_2\). According to the problem, the passenger bus takes 20 hours and the freight train takes 24 hours, so:
\[
V_1 = \frac{D}{20}, \quad V_2 = \frac{D}{24}
\]
where \(D\) is the distance between locations A and B.

\textbf{Distance relationship at the meeting point.} When they meet, the freight train is still 636 kilometers away from location A, which means the passenger bus has already traveled \(D-636\) kilometers. Let \(t\) be the time in hours when they meet. Then:
\[
V_1 t = D-636
\]
\[
V_2 t = 636
\]

\textbf{Substitute the speed expressions:}
\[
\frac{D}{20} t = D-636
\]
\[
\frac{D}{24} t = 636
\]

\textbf{Solve for the unknown.} From \(\frac{D}{24} t = 636\), we get:
\[
t = \frac{636 \times 24}{D}
\]
Substitute \(t\) into \(\frac{D}{20} t = D-636\):
\[
\frac{D}{20} \times \frac{636 \times 24}{D} = D-636
\]
\[
\frac{636 \times 24}{20} = D-636
\]
\[
D = \frac{636 \times 24}{20} + 636
\]
\[
D = 636 \times \frac{24 + 20}{20}
\]
\[
D = 636 \times \frac{44}{20}
\]
\[
D = 636 \times 2.2
\]
\[
D = 1399.2
\]

\textbf{Verify the answer.} Check if \(D = 1399.2\) meets the conditions:
- The passenger bus traveled \(1399.2-636 = 763.2\) kilometers.
- The speed of the passenger bus is \(\frac{1399.2}{20} = 69.96\) km/h.
- The speed of the freight train is \(\frac{1399.2}{24} = 58.3\) km/h.
- The meeting time \(t = \frac{636}{58.3} = 10.92\) hours.
- The passenger bus traveled \(69.96 \times 10.92 = 763.2\) kilometers, and the freight train traveled \(58.3 \times 10.92 = 636\) kilometers, which meets the conditions.

\textbf{Conclusion.} The distance between locations A and B is 1399.2 kilometers.

Final answer: \boxed{1399.2} kilometers.
\end{CASE}

\begin{CASE}[title={The case of question "It takes 20 hours for the bus to go from A to B, and 24 hours for the train to go from B to A. The two vehicles start from opposite directions at the same time, and when they meet, the train is still 636 km away from A. What is the distance between A and B?" at CMCTS}, label=case:cmcts_3]
\textbf{We need to think step by step to understand the meaning of the problem.} The problem states: A passenger bus takes 20 hours to travel from location A to location B, and a freight train takes 24 hours to travel from location B to location A. Both depart simultaneously from their respective locations and travel towards each other. When they meet, the freight train is still 636 kilometers away from location A. We need to find the distance between locations A and B.

\textbf{This problem is quite challenging; we should first analyze the knowledge points involved and the mathematical formulas and properties used.}
1. The meeting time for both the passenger bus and the freight train is the same since they depart simultaneously and meet at the same time.
2. Speed is related to distance and time: speed = distance / time. Therefore, the speed of the passenger bus is the distance between A and B divided by 20 hours, and the speed of the freight train is the distance between A and B divided by 24 hours.
3. When they meet, the freight train has traveled the distance between A and B minus 636 kilometers.

\textbf{Let’s set up the equations and solve them:}
1. Let the distance between locations A and B be \(x\) kilometers, and the speeds of the passenger bus and freight train be \(v_1\) and \(v_2\), respectively.
2. The speed of the passenger bus \(v_1 = \frac{x}{20}\), and the speed of the freight train \(v_2 = \frac{x}{24}\).
3. Let the meeting time be \(t\) hours. The distance traveled by the freight train is \(v_2 \times t = \frac{x}{24} \times t\).
4. According to the problem, \(\frac{x}{24} \times t = x-636\).
5. The distance traveled by the passenger bus is \(v_1 \times t = \frac{x}{20} \times t\).
6. Since the two vehicles meet, the sum of the distances they have traveled equals the distance between locations A and B, i.e., \(v_1 \times t + v_2 \times t = x\).

\textbf{First, solve the first equation for \(t\):}
\[
\frac{x}{24} \times t = x-636
\]
\[
t = \frac{24(x-636)}{x}
\]

\textbf{Substitute \(t\) into the second equation:}
\[
\frac{x}{20} \times \frac{24(x-636)}{x} + \frac{x}{24} \times \frac{24(x-636)}{x} = x
\]
\[
\frac{24(x-636)}{20} + (x-636) = x
\]
\[
\frac{6(x-636)}{5} + (x-636) = x
\]
\[
\frac{6x-3816}{5} + x-636 = x
\]
\[
\frac{6x-3816 + 5x-3180}{5} = x
\]
\[
\frac{11x-6996}{5} = x
\]
\[
11x-6996 = 5x
\]
\[
6x = 6996
\]
\[
x = 1166
\]

\textbf{We need to consider if there are any ambiguities in the problem statement.} We have established equations based on the given conditions and solved them algebraically, concluding that the distance between locations A and B is 1166 kilometers. After careful analysis, this answer is reasonable and there are no ambiguities.

Final answer: The distance between locations A and B is \boxed{1166} kilometers.
\end{CASE}

In  Case  \ref{case:cmcts_1}, the model reflects but doesn't correct previous errors. These cases are quite common, which is why CMCTS-PRM-w/o-reflect shows the least performance drop among all ablation groups in Table \ref{tab:cmp_moudle}. Generally, as with CMCTS in Case  \ref{case:cmcts_1}, additional validation information from the code is needed to reflect accurately.

These cases clearly show that the CMCTS framework achieves significant performance improvements in mathematical reasoning tasks through its constrain action set and process reward model. It not only increases the diversity of candidate states but also enhances the logical reasoning process.
\section{Prompt}
Our work only requires simple prompts to run effectively. Below, I will list the prompts we use in our method. In most cases, these prompts are universal, and more detailed information should be referred to our repository at \url{https://github.com/pass-lin/CMCTS}.

\begin{prompt}[title={Prompt \thetcbcounter: Instruction}, label=prompt:Instruction]
Below is a mathematical problem. Please think step by step and solve it. Enclose each thought process with the <think> and </think> symbols. The thought process should be as rich and detailed as possible, delving into every content and detail deeply, rather than just skimming over,
After you feel that the thought process is sufficient to solve the problem, organize your thought process into a complete answer, and write the final answer in boxed{{}}.\\
\end{prompt}

\begin{prompt}[title={Prompt \thetcbcounter: understand action}, label=prompt:understandaction]
We need to think step by step to understand the meaning of the problem\\
Let's consider if there are any ambiguities in the problem statement\\
This problem is quite difficult; we should first analyze what knowledge points it involves, what mathematical formulas and related properties it utilizes\\
\end{prompt}

\begin{prompt}[title={Prompt \thetcbcounter: code action}, label=prompt:codeaction]
We can write a piece of code to validate our idea\\
We can write a piece of code to assist with the calculation\\
We can write a piece of code to check our calculation results\\
\end{prompt}
\begin{prompt}[title={Prompt \thetcbcounter: reflect action}, label=prompt:reflectaction]
We should step by step check if the above process is reasonable and correct\
There may be errors and inaccuracies in the above questions; we need to step by step check for any mistakes\\
We need to combine the above thought process to see if it aligns with the problem's intention\\
There are some details in the problem that were not considered clearly; we need to check them\\
\end{prompt}

\begin{prompt}[title={Prompt \thetcbcounter: summary action}, label=prompt:summaryaction]
Based on the above thought process, we have solved this problem. Now, let's organize our thoughts into a complete answer and write the final answer in boxed{}. If there are multiple questions, I will answer each one in turn, separated by commas.\\
Now, let's compile our thought process into a complete answer and place the final answer in boxed{}. If there are multiple questions, I will answer each one in turn, separated by commas.\\
\end{prompt}

However, it is important to note that for multiple-choice questions, the corresponding instruction and summary action will be slightly different. We need to encourage the LLM to directly select from the options rather than just providing the calculation result. For specific prompt templates, refer to Prompt \ref{prompt:mul-Instruction} and Prompt \ref{prompt:mul-summary}.

\begin{prompt}[title={Prompt \thetcbcounter: Multiple-choice question Instruction}, label=prompt:mul-Instruction]
You will be given a mathematics problem. Please think through and solve it step by step. Enclose each thought process with <think> and </think> tags. Make your thought process as rich and detailed as possible, deeply considering every content and detail, rather than briefly skimming over them.
Once you believe the thought process is sufficient to solve the problem, organize your thoughts into a complete answer, and choose the option from “A”, “B”, “C”, “D”, “E” that is closest to your answer. Write the final answer in boxed{{}}.

\end{prompt}

\begin{prompt}[title={Prompt \thetcbcounter:Multiple-choice question summary action}, label=prompt:mul-summary]
Based on the above thought process, we have solved this problem. First, we will recall the problem, then organize our thought process and write down the final solution, and finally choose the closest answer from “A”, “B”, “C”, “D”, “E”. The final answer will be written in boxed{{}}.\\
Now, let's first recall the problem. Then, we will organize our thought process into a complete answer, and finally choose the closest answer from “A”, “B”, “C”, “D”, “E”. The final answer will be written in boxed{{}}.\\
\end{prompt}

\bibliography{sn-bibliography}
\end{document}